\documentclass[twoside]{article}

\usepackage[preprint]{aistats2026}

\usepackage[round]{natbib}

\usepackage{amsmath}
\usepackage{amssymb}
\usepackage{algorithm}
\usepackage{algorithmic}
\usepackage{booktabs}
\usepackage{graphicx} 
\usepackage{amsthm}
\usepackage{mathrsfs}
\newcommand{\customsmall}{\fontsize{8.5}{10}\selectfont}

\usepackage{mathtools}
\newtheorem{remark}{Remark}
\newtheorem{theorem}{Theorem}
\newtheorem{lemma}{Lemma}
\usepackage{placeins}
\usepackage{url}

\usepackage[hidelinks]{hyperref}

\begin{document}

\twocolumn[

\aistatstitle{Connectome-Guided Automatic Learning Rates for Deep Networks}
\aistatsauthor{Peilin He \And Tananun Songdechakraiwut}

\aistatsaddress{ Division of Natural and Applied Sciences\\Duke Kunshan University \And Department of Computer Science\\Duke University } ]

\begin{abstract}

The human brain is highly adaptive: its functional connectivity reconfigures on multiple timescales during cognition and learning, enabling flexible information processing. By contrast, artificial neural networks typically rely on manually-tuned learning-rate schedules or generic adaptive optimizers whose hyperparameters remain largely agnostic to a model's internal dynamics. In this paper, we propose Connectome-Guided Automatic Learning Rate (CG-ALR) that dynamically constructs a functional connectome of the neural network from neuron co-activations at each training iteration and adjusts learning rates online as this connectome reconfigures. This connectomics-inspired mechanism adapts step sizes to the network's dynamic functional organization, slowing learning during unstable reconfiguration and accelerating it when stable organization emerges. Our results demonstrate that principles inspired by brain connectomes can inform the design of adaptive learning rates in deep learning, generally outperforming traditional SGD-based schedules and recent methods.

\end{abstract}

\section{Introduction}
Stochastic gradient descent (SGD)~\citep{robbins1951sgd}, RMSProp~\citep{tieleman2012rmsprop}, and Adam~\citep{kingma2015adam} are indispensable optimization tools in modern machine learning. Yet their effectiveness hinges critically on the choice of learning rate. If the step size is too large, updates may become unstable and diverge; if too small, training can stagnate. Moreover, there is no universal ``one-size-fits-all'' schedule, as different regions of the parameter space or different phases of training may require drastically different learning rate magnitudes.

To address this challenge, many handcrafted learning-rate schedules have been developed, such as exponential decay, step decay, cosine annealing, and plateau-based reduction. While effective in some scenarios, these schedules are time-driven or loss-triggered heuristics that do not directly reflect the intrinsic state of the evolving representation. More recently, parameter-free adaptive methods have been introduced, such as Distance-over-Gradients (DoG)~\citep{ivgi2023dog}, which sets the step size as the ratio between the running maximum of the parameter distance from initialization and the root of the accumulated squared gradient norms. While reducing manual tuning, such methods rely solely on first-order, gradient-magnitude-based signals and are agnostic to the evolving topology of internal representations, potentially missing important dynamics in deep networks.

In this work, we propose a novel \emph{Connectome-Guided Automatic Learning Rate} (CG-ALR) framework that leverages persistent homology to analyze the topological reconfiguration of intermediate functional connectomes and adapt step sizes accordingly. Rather than relying on the loss-based triggers or time-driven decay rules, CG-ALR drives the learning rate by persistent-homology-quantified representation change. By applying persistent homology to the evolving connectivity graphs, we obtain a topological time series that summarizes functional changes in the network over time.

We show that monitoring the dynamics of this topological time series, quantified via Wasserstein distances between persistence diagrams~\citep{skraba2020wasserstein}, provides an efficient schedule for adaptively changing the learning rate. Our experiments on image classification benchmarks demonstrate that 
CG-ALR achieves competitive or superior performance to traditional 
SGD-based schedules (SGD-Constant, SGD-Cosine, SGD-Step, SGD-Exp, 
SGD-Plateau~\citep{goodfellow2016deep,pytorch2023lr}). We also conduct direct comparisons with DoG, the recent advanced parameter-free adaptive method and show our approach yields performance competitive with or exceeding it.

Our main contributions are as follows:
\begin{itemize}
    \item{We propose \emph{Connectome-Guided Automatic Learning Rate} (CG-ALR), a topology-aware adaptive mechanism that leverages persistent homology to capture the representational dynamics of the network's changing internal functional connectomes.}
    \item{We design a robust change-detection signal for CG-ALR with hysteresis and a late-phase multiplier, reducing jitter and preventing premature annealing.}
    \item{We provide extensive empirical validation across diverse datasets and architectures, showing that CG-ALR is competitive with or superior to both traditional SGD-based schedules and advanced parameter-free adaptive methods.}
\end{itemize}

\section{Methods}

\subsection{Functional Connectomes}
\label{sec:connectome}

Consider a training set $\mathcal{X} = \{x^{(1)}, \dots, x^{(N)}\}$, where each sample $x^{(i)}$ is passed through a neural network. We focus on a collection of $P$ \emph{activation signals}, which may correspond to neurons in a fully connected layer, channels in a convolutional layer, or even units drawn from multiple layers together. For input $x^{(i)}$ and activation index $j \in \{1,\dots,P\}$, let the scalar response be $s_{ij}$. Collecting these responses forms the matrix
$A = [s_{ij}] \in \mathbb{R}^{N \times P},$
where row $i$ corresponds to sample $x^{(i)}$. The activation profile of signal $j$ across the dataset is then
$\mathbf{a}_j = [s_{1j}, \dots, s_{Nj}]^\top .$

To study relationships among activations, we ask whether two signals tend to vary together across the dataset. This correlation provides a measure of statistical dependency between their activation profiles. Following conventions in neuroscience and connectomics, we adopt the widely-used Pearson correlation coefficient:
$\rho_{ij} = \mathrm{Cov}(\mathbf{a}_i, \mathbf{a}_j) / \big(\sigma(\mathbf{a}_i)\sigma(\mathbf{a}_j)\big) ,$
where $\mathrm{Cov}$ denotes covariance and $\sigma$ denotes standard deviation.
The \emph{functional connectome} is then defined as the weighted adjacency matrix $M \in \mathbb{R}^{P \times P}$, with entries
$M_{ij} = |\rho_{ij}|$ for $i \neq j$ and $M_{ii} = 0$, 
so that $M$ captures pairwise statistical dependencies among activation signals, excluding self-loops.

When the activation signals are drawn from a fully connected layer, each corresponds to a neuron applying an affine transformation followed by a nonlinearity. Its activation profile $\mathbf{a}_j$ is simply the vector of activations across samples.

This construction extends naturally to other architectures. For example, in a convolutional layer, each activation signal corresponds to a filter producing activation maps 
$z_{ij} \in \mathbb{R}^{H \times W}$ 
for sample $x^{(i)}$, where $H$ and $W$ are the spatial dimensions after convolution. A reduction operator (e.g., mean pooling, max pooling, or global norm) maps each $z_{ij}$ to a scalar $s_{ij}$, yielding activation profiles $\mathbf{a}_j$ from which correlations are computed as before.

\subsection{Topological Connectome Distances}
\label{sec:distances}

Given two functional connectomes $M^{(t)}, M^{(t+1)} \in \mathbb{R}^{P \times P}$ obtained at consecutive training epochs, we quantify their dissimilarity using \emph{persistent homology} (PH) and its graph-based variant, \emph{persistent graph homology} (PGH)~\citep{songdechakraiwut2021topological,songdechakraiwut2023topological}. These methods capture higher-order topological features of connectomes (e.g., loops and components) beyond pairwise edge comparisons, and the stability theorem guarantees robustness to small perturbations \citep{skraba2020wasserstein}. We treat each connectome as a weighted undirected graph and define a distance $D(M^{(t)}, M^{(t+1)})$ from topological signatures of their edge weights, measuring how much the network's functional connectivity patterns change from one epoch to the next.  

Formally, the \emph{$p$-Wasserstein distance} between two persistence objects $X$ and $Y$ (e.g., diagrams or vectors) is defined as
\[
W_p(X, Y) = \left( \inf_{\gamma \in \Gamma(X, Y)} \sum_{(x, y) \in \gamma} \|x - y\|^p \right)^{1/p},
\]
where $X$ and $Y$ are multisets of points and $\Gamma(X, Y)$ is the set of all valid matchings, including optional diagonal projections for unmatched features.

\noindent\textbf{PGH.}
Starting from the adjacency matrix $M^{(t)}$, we extract its maximum spanning tree $T$. The remaining edges outside $T$ form 1-cycles, whose weights are regarded as death times. Sorting these values yields a persistence vector $\mathbf{d}^{(t)} = [d^{(t)}_1 \leq \cdots \leq d^{(t)}_k]$~\citep{songdechakraiwut2023wasserstein}. The distance between consecutive persistence vectors, referred to as $\mathrm{TOP}$, is defined using the 1-Wasserstein distance~\citep{songdechakraiwut2025functional,wu2025data}:
\[
\mathrm{TOP}(\mathbf{d}^{(t)}, \mathbf{d}^{(t+1)}) = \sum_{j=1}^{k} \left| d_j^{(t)} - d_j^{(t+1)} \right|.
\]
Thus, $D(M^{(t)}, M^{(t+1)}) = \mathrm{TOP}(\mathbf{d}^{(t)}, \mathbf{d}^{(t+1)})$ provides one way to compare connectomes.

\noindent\textbf{PH.} 
We perform a Vietoris--Rips filtration on the dissimilarity matrix $1 - M^{(t)}$, yielding a persistence diagram $\mathrm{PD}^{(t)} = \{(b_i^{(t)}, d_i^{(t)})\}_i$, where each point $(b_i^{(t)}, d_i^{(t)})$ represents the birth and death of a topological feature (e.g., connected component or loop). The distance between consecutive diagrams, denoted $\mathrm{WD}$, is defined as the 2-Wasserstein distance:
\begin{align*}
    \mathrm{WD}\big(&\mathrm{PD}^{(t)}, \mathrm{PD}^{(t+1)}\big) \\
    &= \left( \inf_{\gamma} \sum_{i} 
       \big\| (b_i^{(t)}, d_i^{(t)}) - \gamma(b_i^{(t+1)}, d_i^{(t+1)}) \big\|_2^2 \right)^{1/2}.
\end{align*}
Thus, $D(M^{(t)}, M^{(t+1)}) = \mathrm{WD}(\mathrm{PD}^{(t)}, \mathrm{PD}^{(t+1)})$ provides another way to compare connectomes.

Beyond $\mathrm{TOP}$ and $\mathrm{WD}$, we also consider the bottleneck distance (BD)~\citep{cohen2007stability}, the heat kernel distance (HK)~\citep{reininghaus2015stable}, and the sliced Wasserstein kernel (SWK)~\citep{carriere2017sliced}. Implementation details are provided in the Appendix.

\subsection{Topological Connectome Dynamics}
\label{sec:time_series}

We construct a time series of distances to capture the dynamics of functional connectomes during training. 
Let the number of training epochs be $t = 1, 2, \ldots, T-1$. 
At each step, we compute
$\delta_t = D(M^{(t)}, M^{(t+1)}),$
which quantifies the topological change from epoch $t$ to $t+1$.

To smooth fluctuations and emphasize underlying trends, we apply an exponential moving average:
\begin{equation*}
\tilde{\delta}_t = (1-\lambda)\, \delta_t + \lambda \tilde{\delta}_{t-1}, 
\quad \tilde{\delta}_1 = \delta_1,
\label{eq:ema}
\end{equation*}
with smoothing factor $\lambda \in [0,1)$.

Finally, we normalize the smoothed distances using a rolling median--MAD scaling, which recenters the signal and rescales it by a robust measure of variability, allowing stable thresholding across epochs and runs. 
Here, the statistics are computed using all epochs observed up to time $t$:
\begin{equation}
z_t = \frac{\tilde{\delta}_t - \mathrm{median}(\{\tilde{\delta}_u\}_{u=1}^{t})}
{\mathrm{MAD}(\{\tilde{\delta}_u\}_{u=1}^{t}) + \tau},
\label{eq:toposignal}
\end{equation}
where $\tau > 0$ is a small constant for numerical stability and
\[
\mathrm{MAD}\!\left(\{\tilde{\delta}_u\}_{u=1}^{t}\right) 
= \mathrm{median}\big(|\tilde{\delta}_u - \mathrm{median}(\{\tilde{\delta}_r\}_{r=1}^{t})|\big).
\]
The resulting normalized signal $z_t$ provides a compact and robust measure that adapts online, and we later use it in Section~\ref{sec:CG-ALR} to guide adaptive learning rate updates.

\subsection{Connectome-Guided Adaptive Learning Rate}
\label{sec:CG-ALR}

\begin{algorithm}[t]
\small
\caption{Online Connectome-Guided Learning Rate}
\label{alg:learning_rate}
\begin{algorithmic}[1]
\REQUIRE initial learning rate $\eta^\star$; RM envelope $(t_0\!\ge\!1,\ \alpha\!\in(\tfrac12,1])$; global steps counter $s$; multipliers $(\gamma_{\downarrow},\gamma_{\uparrow},\gamma_{\mathrm{late}})$; late split $N_{\mathrm{late}}$;
epochs $T$; batches per epoch $B$.

\ENSURE learning-rate schedule $\{\eta_{t,b}\}$; controller $\{\psi_t\}$

\STATE $\psi_0 \gets 1$; $s \gets 0$;
\STATE $\eta_0 \gets \eta^\star  \cdot t_0^{\alpha}$  
\FOR{$t = 1$ to $T$}
  \FOR{$b = 1$ to $B$}
    \STATE $\bar\eta \gets \eta_0 / (s+t_0)^{\alpha}$ \textit{// RM envelope}
    \STATE $\eta_{t,b} \gets \bar\eta \cdot \psi_{t-1}$
    \STATE $s \gets s+1$
  \ENDFOR
  \STATE $z_t \gets \mathrm{ComputeTopoSignal}(t)$ \ \ \textit{// Eq.~\eqref{eq:toposignal}}
  \IF{still in early epochs}
    \STATE $u_t \gets 1$
  \ELSE
    \STATE  $\varepsilon_t \gets \mathrm{AdaptiveThreshold}(t)$ \textit{// Eq.~\eqref{eq:epsilon_t}}
    \IF{cooldown active}
        \STATE $u_t \gets 1$
    \ELSIF{$z_t > \varepsilon_t$}
      \STATE $u_t \gets \gamma_{\downarrow}$ \textit{// Aggressive downscale}   
    \ELSE
        \IF{$t \le N_{\mathrm{late}}$}
            \STATE $u_t \gets \gamma_{\uparrow}$ \textit{// Moderate upscale}
        \ELSE
            \STATE $u_t \gets \gamma_{\mathrm{late}}$ \textit{// Milder downscale in late phase}
        \ENDIF
    \ENDIF
  \ENDIF
  \STATE $\psi_t \gets \mathrm{clip}(\psi_{t-1}\cdot u_t,\ \psi_{\min},\ \psi_{\max})$
\ENDFOR
\RETURN $\{\eta_{t,b}\},\ \{\psi_t\}$

\end{algorithmic}
\end{algorithm}

Conventional learning-rate schedules rely on loss or accuracy curves, which are insensitive to changes in the internal representation topology. To address this limitation, we propose \emph{Connectome-Guided Adaptive Learning Rate} (CG-ALR), which use the normalized connectome topological signal $z_t$ as feedback to adapt the learning rate in a stable yet responsive manner.

Large $z_t$ values signal unstable representation topology, warranting slower learning to avoid destabilization. Small values indicate stability, permitting faster learning supported by more reliable gradients. We therefore update the next-step rate by
\begin{equation}
\eta_{t+1} =
\begin{cases}
\gamma_{\downarrow}\,\eta_t, & \text{if } z_t > \varepsilon_t, \\
\gamma_{\uparrow}\,\eta_t, & \text{if } z_t \le \varepsilon_t \ \text{and } t \le N_{\mathrm{late}}, \\
\gamma_{\mathrm{late}}\,\eta_t, & \text{if } z_t \le \varepsilon_t \ \text{and } t > N_{\mathrm{late}}.
\end{cases}
\label{eq:eta-update}
\end{equation}
The adaptive threshold is defined by
\begin{equation}
\varepsilon_t = \mathrm{median}(\{z_u\}_{u=1}^t) 
+ k_{\mathrm{MAD}} \cdot \mathrm{MAD}(\{z_u\}_{u=1}^t),
\label{eq:epsilon_t}
\end{equation}
where $k_{\mathrm{MAD}}>0$ controls the sensitivity. This threshold adapts through a cumulative median--MAD statistic, ensuring robustness by responding to relative rather than absolute fluctuations in topology. Here, $\gamma_{\downarrow}\!<\!1$ is the downscale factor, $\gamma_{\uparrow}\!>\!1$ is the upscale factor, and in the late phase we apply an additional decay $\gamma_{\mathrm{late}}\in(\gamma_{\downarrow},1)$ to encourage steady convergence.

Equivalently, the update rules in Eq. \eqref{eq:eta-update} can be expressed using a multiplicative controller $\psi_t$:
\[
\psi_{t} = \mathrm{clip}\!\left(\psi_{t-1} \cdot u_t,\; \psi_{\min},\, \psi_{\max}\right),
\]
with $u_t \in \{\gamma_{\downarrow},\, \gamma_{\uparrow},\, \gamma_{\mathrm{late}},\, 1\}$, and $\mathrm{clip}(x,a,b)=\min\{b,\max\{a,x\}\}$. 
The baseline schedule $\bar{\eta}(s)$ evolves at the batch level, while the controller $\psi_t$ evolves at the epoch level. The effective learning rate for batch $b$ in epoch $t$ is
\[
\eta_{t,b} = \bar{\eta}(s) \cdot \psi_{t-1},
\]
where $s$ is the cumulative batch index (i.e., $s = tB+b$ if there are $B$ batches per epoch). The Robbins--Monro (RM) baseline \citep{robbins1951sgd} is
\[
\bar{\eta}(s) = \frac{\eta_0}{(s+t_0)^\alpha}, \quad \alpha \in \left(\tfrac{1}{2},1\right],
\]
with $\eta_0 = \eta^\star \cdot t_0^\alpha$ chosen so that the initial step size is calibrated relative to the target $\eta^\star$. Thus, $\bar{\eta}(s)$ provides fine-grained decay across mini-batches, while $\psi_t$ introduces slower, topology-driven adjustments across epochs.
To further avoid chattering, we require that the condition on $z_t$ hold for several consecutive epochs before triggering an update, and we impose a cooldown period of a few epochs during which no further adjustments are allowed. 

By coupling the optimizer with the topology-driven signal $z_t$, CG-ALR adapts step sizes by slowing learning during unstable reconfiguration and accelerating it when stable organization emerges, improving both training stability and generalization (see Algorithm~\ref{alg:learning_rate}). Reference code implementing CG-ALR is provided in the supplementary material.

\begin{theorem}
Our controller preserves the Robbins--Monro envelope, and thus the proposed CG-ALR algorithm enjoys the same convergence guarantees: the iterates almost surely approach stationary points, and under PL/KL conditions converge to minimizers. 
\end{theorem}
\begin{proof}
See Appendix for the complete proof.
\end{proof}

\section{Experiments}
\paragraph{Datasets.}
We conducted experiments on six benchmark datasets. For images, we use CIFAR-10, CIFAR-100 \citep{krizhevsky2009learning}, and Mini-ImageNet \citep{vinyals2016matching}, which contain natural images spanning multiple object categories. For graphs, we use MUTAG, PROTEINS, and ENZYMES \citep{Morris2020TUDataset}, which involve molecular and protein structures for classification. Further dataset details are provided in the Appendix.

\paragraph{Architecture and Optimization.}
For image datasets (CIFAR-10/100 and Mini-ImageNet), we used a ResNet-18~\citep{he2016deep} backbone with a three-layer fully connected (FC) classifier head; the configuration is summarized in Table~\ref{tab:architectures_resnet}. For graph datasets (MUTAG, PROTEINS, ENZYMES), we used three graph convolutional layers with GCN~\citep{kipf2016semi} and GAT~\citep{velivckovic2017graph} backbones, followed by pooled features fed into a FC head with one, two, or three layers to study the effect of classifier depth; the configuration is summarized in Table~\ref{tab:architectures_gnn}.

\begin{table}[t]
\small
\centering
\setlength{\tabcolsep}{4.3mm}
\caption{Model Architecture for CIFAR-10/100 and Mini-ImageNet. Each layer is shown as [\emph{input dimension, output dimension}].}
\label{tab:architectures_resnet}
\vspace{1.5em}
\begin{tabular}{lc}
\toprule
Layer & \textbf{CIFAR-10/100/Mini-ImageNet} \\
\midrule
Backbone & ResNet-18\\
\midrule
Flatten & [512]\\
\midrule
FC1 & [512,\,512]\\
FC2 & [512,\,256]\\
FC3 & [256,\,$C$]\\
\midrule
Output & $C=10/100/100$ \\
\bottomrule
\end{tabular}
\end{table}

\begin{table}[t]
\centering
\small
\setlength{\tabcolsep}{0.1mm}
\caption{Model Architectures for Graph Datasets (MUTAG, PROTEINS, ENZYMES). Each layer is shown as [\emph{input dimension, output dimension}].}
\label{tab:architectures_gnn}
\vspace{1.5em}
\begin{tabular}{lcc}
\toprule
Layer & \textbf{MUTAG} & \textbf{PROTEINS/ENZYMES} \\
\midrule
Backbone & GCN/GAT & GCN/GAT \\
\midrule
Pooling & [64] & [128]  \\
\midrule
FC1 & [64,\,64]/[64,\,$C$] & [128,\,128]/[128,\,$C$] \\
FC2 (optional) & [64,\,64]/[64,\,$C$] & [128,\,128]/[128,\,$C$]  \\
FC3 (optional) & [64,\,$C$] & [128,\,$C$] \\
\midrule
Output & $C=2$ & $C=2/6$  \\
\bottomrule
\end{tabular}
\end{table}

\subsection{Study 1: Dynamics of Adaptive Learning Rates}
\label{sec:study1}
\begin{figure*}[t]
\centering
\includegraphics[width=.67\linewidth]{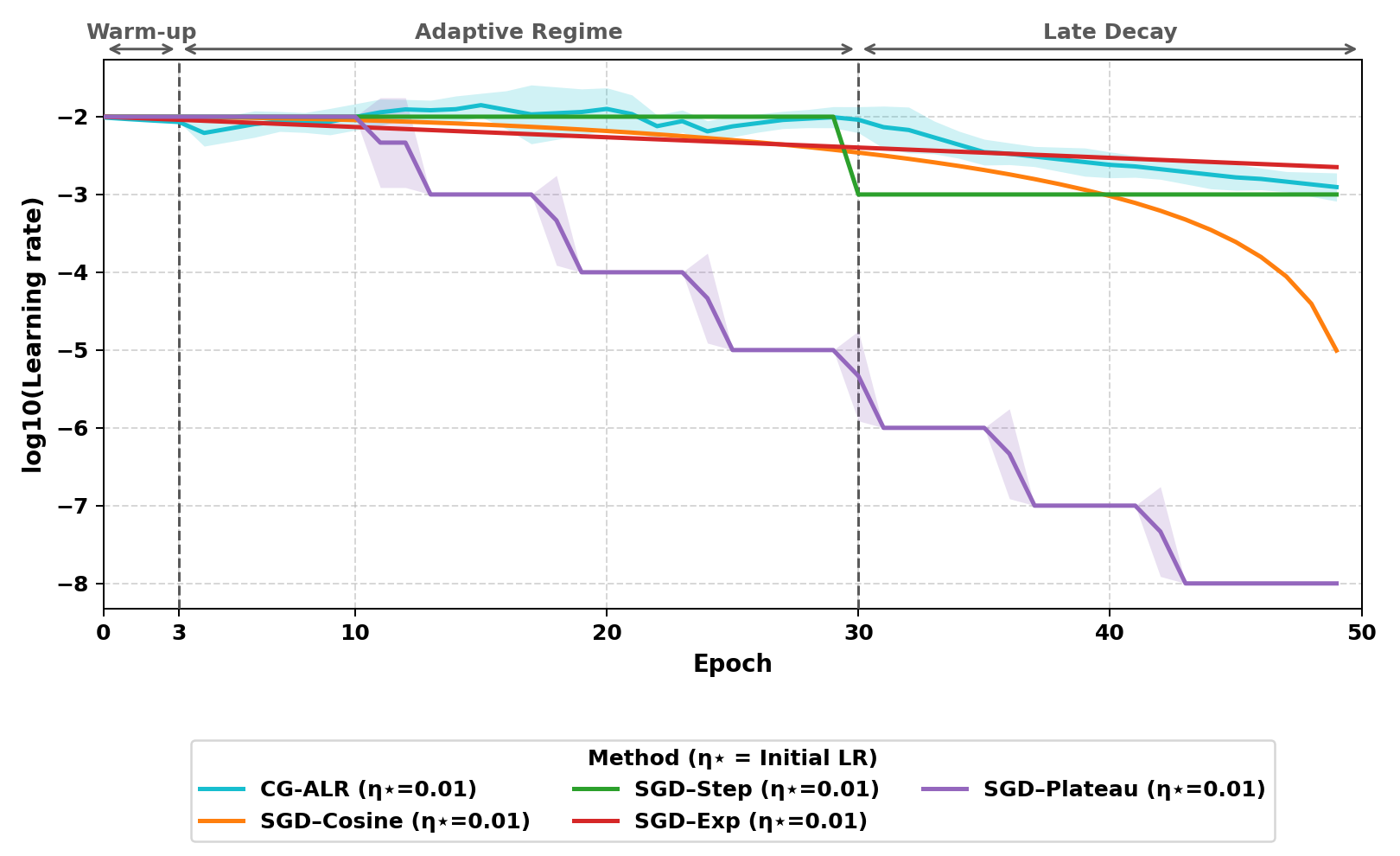}
\caption{Learning Rate Dynamics Across Epochs on CIFAR-10. Comparison of our proposed CG-ALR with TOP against standard schedules that require initial learning rates: SGD-Cosine, SGD-Step, SGD-Exp, and SGD-Plateau.}
\label{fig:study1_CIFAR-10}
\end{figure*}

This study analyzes the trajectory of adaptive learning rate dynamics with an initial value of $\eta^\star = 0.01$. In this study, we incorporate TOP into our CG-ALR and compare it against four standard SGD-based schedules that also require initial learning rates. As shown in Figure~\ref{fig:study1_CIFAR-10}, SGD-Step and SGD-Plateau reduce the learning rate by orders of magnitude early in training, while SGD-Cosine and SGD-Exp shrink it monotonically toward minimal values, slowing late-stage progress. In contrast, CG-ALR tracks representation reorganization via the connectome signal and produces a gradual, on-demand trajectory, with the learning rate remaining flat during the warm-up stage (first 3 epochs), adapting in the intermediate stage, and decaying only in the late stage (after 30 epochs), thereby avoiding premature annealing.

\begin{figure*}[t]
  \centering
  \includegraphics[width=\linewidth]{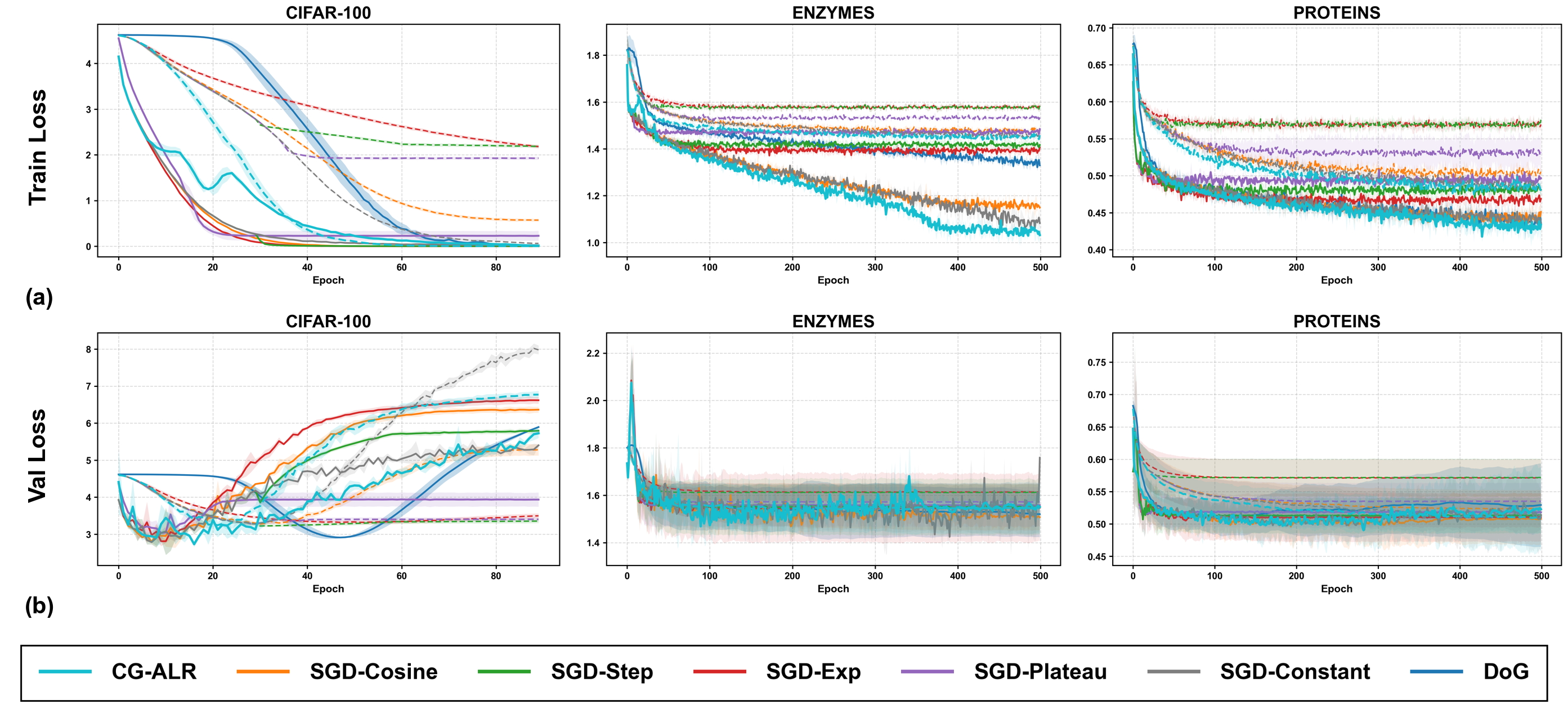}
  \vspace{0.6ex}
  \caption{
  (a) Training Loss on Three Datasets (\textbf{CIFAR-100}, \textbf{ENZYMES}, \textbf{PROTEINS}).
  Solid lines are the averages across seeds for the \emph{best} learning-rate choice from the predefined grid ($\eta^\star \!\in\!\{0.1,0.01,0.001\}$ for CG-ALR and SGD-based schedules; $\eta\!\in\!\{0.1,0.01,0.001\}$ for SGD-Constant; default hyperparameters for DoG).
  Dashed lines are the averages across seeds for the \emph{worst} choice from the same grid.
  Shaded bands denote the min–max range across three seeds.
  For CG-ALR, we use \textbf{TOP} on CIFAR-100 and \textbf{SWK} on ENZYMES/PROTEINS.
  (b) Validation Loss on the Same Three Datasets for the Same Set of Optimizers.
  }
  \label{fig:study2_val_train_loss}
\end{figure*}

\subsection{Study 2: Ablation Analyses}

In this study, we conduct ablation analyses to understand how design choices affect the performance of our CG-ALR. We focus on two important factors: (i) the number of FC layers used in the classifier head; and (ii) the size of the probe set used to compute the connectome, defined as a stratified subset of the training samples and denoted $P_{\text{probe}}$.

\paragraph{Effect of FC Layer Depth.}
\label{sec:layer_extraction}
We vary how connectomes are extracted by adjusting the number of FC layers used in the GCN model. To compare depths, we aggregate accuracy, generalization stability, and convergence efficiency into a single composite score. For each dataset and FC depth $d\in\{1,2,3\}$, let $\overline{m}_d$ denote the mean of metric $m$ across CG-ALR methods, seeds, and initial learning-rate values. We standardize each metric across depths within a dataset as
\[
z(\overline{m}_d)=
\frac{\overline{m}_d-\frac{1}{3}\sum_{d'=1}^{3}\overline{m}_{d'}}
{\sqrt{\frac{1}{3}\sum_{d'=1}^{3}\left(\overline{m}_{d'}-\frac{1}{3}\sum_{d''=1}^{3}\overline{m}_{d''}\right)^2}},
\]
with $z(\overline{m}_d)=0$ if the denominator is zero. The composite score for depth $d$ is then
\[
\mathrm{Composite}(d)=
\sum_{p\in\mathcal{P}} z(\overline{p}_d)
-\sum_{g\in\mathcal{G}} z(\overline{g}_d)
-\sum_{c\in\mathcal{C}} z(\overline{c}_d),
\]
where
\begin{itemize}
\item $\mathcal{P}$ (performance): test accuracy at best-validation epoch, test accuracy at final epoch, and maximum validation accuracy;
\item $\mathcal{G}$ (generalization): absolute test–val gap at best-validation epoch, test–val gap at final epoch, train–val gap (accuracy) at final epoch, train–val gap (loss) at final epoch;
\item $\mathcal{C}$ (convergence): wall-clock time to reach max validation accuracy.
\end{itemize}
Higher values indicate better overall performance (higher accuracy, smaller generalization gaps, and faster convergence).

Across datasets, Table~\ref{tab:composite} shows that a single FC layer is optimal for MUTAG, while two FC layers consistently perform best on ENZYMES and PROTEINS. Adding a third FC layer does not provide additional gains, indicating that efficient training does not require deeper stacks. Instead, constructing the target connectome with the appropriate FC depth is sufficient to capture representational dynamics and enable strong, efficient performance.

\begin{table}[t]
\small
\centering
\setlength{\tabcolsep}{1.8mm}
\caption{Composite Scores Across Datasets and FC Layers. Higher values mean better overall performance; \textbf{bold} marks the best score for each dataset.}
\label{tab:composite}
\vspace{1.5em}
\begin{tabular}{lccc}
\toprule
& \multicolumn{3}{c}{\textbf{FC Layer(s) Depth}} \\
\cmidrule(lr){2-4}
Dataset & 1 & 2  & 3  \\

\midrule
MUTAG    &  \textbf{7.043} & 3.957 &  -10.999 \\
ENZYMES  & -3.290        & \textbf{4.721} &  -1.432 \\
PROTEINS & -2.160         & \textbf{3.383} & -1.224\\
\bottomrule
\end{tabular}
\end{table}

\begin{table}[t]
\small
\centering
\setlength{\tabcolsep}{2mm}
\caption{Composite Scores Across Probe Set Sizes for CIFAR-10. Higher values mean better overall performance; \textbf{bold} marks the best score.}
\label{tab:composite_probe}
\vspace{1.5em}
\begin{tabular}{lcccc}
\toprule
& \multicolumn{4}{c}{\textbf{Probe Set Size}} \\
\cmidrule(lr){2-5}
Dataset & 64 & 256 & 1024 & 2096 \\
\midrule
CIFAR-10 & -2.50 & -1.55 & \textbf{7.55} & -3.50 \\
\bottomrule
\end{tabular}
\end{table}

\paragraph{Effect of Probe Set Size.}
\label{sec:probe_set}

We investigate the role of probe set size in CG-ALR, as using more samples to construct the connectomes improves the precision of the topology-change signal $z_t$, but also increases computational cost. To compare sizes, we fix the default training configuration (ResNet-18 backbone with FC depth $d{=}3$), and vary only the probe set size $P_{\text{probe}} \in \{64,\,256,\,1024,\,2096\}$. We then use the same composite score to compare the performance on different probe set sizes. As shown in Table~\ref{tab:composite_probe}, the performance reaches a clear peak when $P_{\text{probe}=1024}$. This indicates that efficient training does not require a larger probe set. Based on this, we set  $P_{\text{probe}=1024}$ as the default in subsequent experiments.

\begin{table*}[t]
\centering
\caption{Median RED with 95\% Percentile-Bootstrap CIs Across Seeds. 
\textbf{Bold} entries indicate cases where CG-ALR achieves better generalization performance ($\mathrm{RED}<0$), 
or statistically significant improvements when the 95\% CI lies entirely below 0. 
Results are shown side by side for CIFAR-10 and CIFAR-100.}
\vspace{1.5em}
\label{tab:study2_c10_c100}

\customsmall
\setlength{\tabcolsep}{9pt}
\begin{tabular}{@{}llrr@{}}
\toprule
\multicolumn{1}{c}{} & \multicolumn{1}{c}{} &
\multicolumn{1}{c}{\textbf{CIFAR-10}} &
\multicolumn{1}{c}{\textbf{CIFAR-100}} \\
\cmidrule(lr){3-4}
\multicolumn{1}{c}{\textbf{CG-ALR}} &
\multicolumn{1}{c}{\textbf{Baseline}} & \multicolumn{1}{c}{RED (95\% CI)} & \multicolumn{1}{c}{RED (95\% CI)} \\
\midrule
BD  & DoG          & \textbf{-0.255}~(\,-0.421,\;0.046\,) & \textbf{-0.143}~(\,\textbf{-0.192},\,\textbf{-0.015}\,) \\
BD  & SGD-Cosine   & \textbf{-0.125}~(\,\textbf{-0.238},\;\textbf{-0.032}\,) & \textbf{-0.096}~(\,\textbf{-0.175},\,\textbf{-0.009}\,) \\
BD  & SGD-Exp      & \textbf{-0.143}~(\,\textbf{-0.226},\;\textbf{-0.028}\,) & \textbf{-0.141}~(\,\textbf{-0.180},\,\textbf{-0.021}\,) \\
BD  & SGD-Plateau  & \textbf{-0.165}~(\,\textbf{-0.194},\;\textbf{-0.023}\,) & \textbf{-0.107}~(\,\textbf{-0.150},\,\textbf{-0.013}\,) \\
BD  & SGD-Step     & \textbf{-0.102}~(\,-\textbf{0.211},\;\textbf{-0.021}\,) & \textbf{-0.090}~(\,\textbf{-0.157},\,\textbf{-0.021}\,) \\
BD  & SGD-Constant & \textbf{-0.142}~(\,\textbf{-0.229},\;\textbf{-0.101}\,) & \textbf{-0.107}~(\,\textbf{-0.164},\,\textbf{-0.058}\,) \\
HK  & DoG          & \textbf{-0.202}~(\,-0.448,\;0.025\,) & \textbf{-0.135}~(\,\textbf{-0.186},\,\textbf{-0.010}\,) \\
HK  & SGD-Cosine   & \textbf{-0.094}~(\,\textbf{-0.260},\;\textbf{-0.047}\,) & \textbf{-0.075}~(\,\textbf{-0.177},\,\textbf{-0.006}\,) \\
HK  & SGD-Exp      & \textbf{-0.106}~(\,\textbf{-0.225},\;\textbf{-0.068}\,) & \textbf{-0.121}~(\,\textbf{-0.190},\,\textbf{-0.012}\,) \\
HK  & SGD-Plateau  & \textbf{-0.122}~(\,\textbf{-0.213},\;\textbf{-0.058}\,) & \textbf{-0.093}~(\,\textbf{-0.136},\,\textbf{-0.002}\,) \\
HK  & SGD-Step     & \textbf{-0.074}~(\,\textbf{-0.233},\;\textbf{-0.055}\,) & \textbf{-0.073}~(\,\textbf{-0.147},\,\textbf{-0.018}\,) \\
HK  & SGD-Constant & \textbf{-0.120}~(\,\textbf{-0.252},\;\textbf{-0.058}\,) & \textbf{-0.081}~(\,\textbf{-0.146},\,\textbf{-0.034}\,) \\
SWK & DoG          & \textbf{-0.241}~(\,\textbf{-0.415},\;\textbf{-0.014}\,) & \textbf{-0.153}~(\,\textbf{-0.189},\,\textbf{-0.012}\,) \\
SWK & SGD-Cosine   & \textbf{-0.112}~(\,\textbf{-0.216},\;\textbf{-0.027}\,) & \textbf{-0.096}~(\,\textbf{-0.182},\,\textbf{-0.006}\,) \\
SWK & SGD-Exp      & \textbf{-0.124}~(\,\textbf{-0.289},\;\textbf{-0.023}\,) & \textbf{-0.141}~(\,\textbf{-0.188},\,\textbf{-0.021}\,) \\
SWK & SGD-Plateau  & \textbf{-0.152}~(\,\textbf{-0.218},\;\textbf{-0.019}\,) & \textbf{-0.110}~(\,\textbf{-0.158},\,\textbf{-0.010}\,) \\
SWK & SGD-Step     & \textbf{-0.081}~(\,\textbf{-0.237},\;\textbf{-0.017}\,) & \textbf{-0.096}~(\,\textbf{-0.163},\,\textbf{-0.019}\,) \\
SWK & SGD-Constant & \textbf{-0.133}~(\,\textbf{-0.256},\;\textbf{-0.102}\,) & \textbf{-0.107}~(\,\textbf{-0.171},\,\textbf{-0.052}\,) \\
TOP & DoG          & \textbf{-0.184}~(\,-0.410,\;0.025\,) & \textbf{-0.153}~(\,\textbf{-0.189},\,\textbf{-0.006}\,) \\
TOP & SGD-Cosine   & \textbf{-0.078}~(\,\textbf{-0.216},\;\textbf{-0.041}\,) & \textbf{-0.096}~(\,\textbf{-0.160},\,\textbf{-0.027}\,) \\
TOP & SGD-Exp      & \textbf{-0.101}~(\,\textbf{-0.289},\;\textbf{-0.019}\,) & \textbf{-0.141}~(\,\textbf{-0.166},\,\textbf{-0.021}\,) \\
TOP & SGD-Plateau  & \textbf{-0.105}~(\,\textbf{-0.218},\;\textbf{-0.015}\,) & \textbf{-0.110}~(\,\textbf{-0.136},\,\textbf{-0.010}\,) \\
TOP & SGD-Step     & \textbf{-0.059}~(\,\textbf{-0.237},\;\textbf{-0.013}\,) & \textbf{-0.096}~(\,\textbf{-0.144},\,\textbf{-0.015}\,) \\
TOP & SGD-Constant & \textbf{-0.117}~(\,\textbf{-0.256},\;\textbf{-0.040}\,) & \textbf{-0.107}~(\,\textbf{-0.149},\,\textbf{-0.034}\,) \\
WD  & DoG          & \textbf{-0.241}~(\,-0.415,\;0.025\,) & \textbf{-0.153}~(\,\textbf{-0.189},\,\textbf{-0.015}\,) \\
WD  & SGD-Cosine   & \textbf{-0.112}~(\,\textbf{-0.265},\;\textbf{-0.027}\,) & \textbf{-0.096}~(\,\textbf{-0.182},\,\textbf{-0.006}\,) \\
WD  & SGD-Exp      & \textbf{-0.124}~(\,\textbf{-0.289},\;\textbf{-0.023}\,) & \textbf{-0.141}~(\,\textbf{-0.188},\,\textbf{-0.021}\,) \\
WD  & SGD-Plateau  & \textbf{-0.152}~(\,\textbf{-0.218},\;\textbf{-0.019}\,) & \textbf{-0.110}~(\,\textbf{-0.136},\,\textbf{-0.010}\,) \\
WD  & SGD-Step     & \textbf{-0.081}~(\,\textbf{-0.237},\;\textbf{-0.017}\,) & \textbf{-0.096}~(\,\textbf{-0.163},\,\textbf{-0.019}\,) \\
WD  & SGD-Constant & \textbf{-0.133}~(\,\textbf{-0.189},\;\textbf{-0.095}\,) & \textbf{-0.107}~(\,\textbf{-0.146},\,\textbf{-0.052}\,) \\
\bottomrule
\end{tabular}
\end{table*}

\begin{table*}[t]
\centering
\caption{Median RED with 95\% Percentile-Bootstrap CIs Across Seeds. 
\textbf{Bold} entries indicate cases where CG-ALR achieves better generalization performance ($\mathrm{RED}<0$), 
or statistically significant improvements when the 95\% CI lies entirely below 0. 
Results are shown for GCN and GAT, with ENZYMES and PROTEINS reported beneath each method.}
\label{tab:study2_enz_pro_gcn_gat_merged}
\vspace{1.5em}
\customsmall
\setlength{\tabcolsep}{2pt}
\begin{tabular}{@{}llrrrr@{}}
\toprule
\multicolumn{2}{c}{} & \multicolumn{2}{c}{\textbf{GCN}} & \multicolumn{2}{c}{\textbf{GAT}} \\
\cmidrule(lr){3-4}\cmidrule(lr){5-6}
\multicolumn{2}{c}{} &
\multicolumn{1}{c}{\textbf{ENZYMES}} &
\multicolumn{1}{c}{\textbf{PROTEINS}} &
\multicolumn{1}{c}{\textbf{ENZYMES}} &
\multicolumn{1}{c}{\textbf{PROTEINS}} \\
\cmidrule(lr){3-6}
\multicolumn{1}{c}{\textbf{CG-ALR}} &
\multicolumn{1}{c}{\textbf{Baseline}} &
\multicolumn{1}{c}{RED (95\% CI)} &
\multicolumn{1}{c}{RED (95\% CI)} &
\multicolumn{1}{c}{RED (95\% CI)} &
\multicolumn{1}{c}{RED (95\% CI)} \\
\midrule
BD  & DoG          & 0.090~(\,0.060,\;0.218\,) & \textbf{-0.014}~(\,-0.104,\;0.053\,) & 0.106~(\,0.013,\;0.248\,) & 0.037~(\,-0.053,\;0.114\,) \\
BD  & SGD-Cosine   & \textbf{-0.034}~(\,-0.119,\;0.026\,) & 0.011~(\,-0.014,\;0.052\,) & 0.000~(\,-0.052,\;0.082\,) & \textbf{-0.016}~(\,-0.200,\;0.049\,) \\
BD  & SGD-Exp      & \textbf{-0.167}~(\,\textbf{-0.230},\;\textbf{-0.090}\,) & \textbf{-0.095}~(\,\textbf{-0.125},\;\textbf{-0.061}\,) & \textbf{-0.099}~(\,\textbf{-0.190},\;\textbf{-0.040}\,) & \textbf{-0.037}~(\,-0.233,\;0.024\,) \\
BD  & SGD-Plateau  & \textbf{-0.128}~(\,\textbf{-0.304},\;\textbf{-0.022}\,) & \textbf{-0.050}~(\,-0.125,\;0.000\,) & \textbf{-0.110}~(\,-0.190,\;0.000\,) & \textbf{-0.071}~(\,-0.233,\;0.000\,) \\
BD  & SGD-Step     & \textbf{-0.192}~(\,\textbf{-0.299},\;\textbf{-0.101}\,) & \textbf{-0.070}~(\,-0.104,\;0.000\,) & \textbf{-0.110}~(\,-0.222,\;0.000\,) & \textbf{-0.098}~(\,\textbf{-0.200},\;\textbf{-0.012}\,) \\
BD  & SGD-Constant & 0.000~(\,-0.026,\;0.095\,) & 0.013~(\,-0.026,\;0.056\,) & 0.040~(\,-0.012,\;0.047\,) & 0.000~(\,-0.037,\;0.040\,) \\
HK  & DoG          & 0.088~(\,0.042,\;0.203\,) & \textbf{-0.082}~(\,-0.214,\;0.014\,) & 0.088~(\,0.035,\;0.248\,) & 0.037~(\,0.014,\;0.132\,) \\
HK  & SGD-Cosine   & \textbf{-0.054}~(\,-0.155,\;0.000\,) & \textbf{-0.032}~(\,-0.061,\;0.011\,) & 0.000~(\,-0.023,\;0.049\,) & 0.011~(\,-0.100,\;0.047\,) \\
HK  & SGD-Exp      & \textbf{-0.167}~(\,\textbf{-0.282},\;\textbf{-0.038}\,) & \textbf{-0.125}~(\,\textbf{-0.214},\;\textbf{-0.094}\,) & \textbf{-0.085}~(\,\textbf{-0.190},\;\textbf{-0.010}\,) & \textbf{-0.023}~(\,-0.139,\;0.012\,) \\
HK  & SGD-Plateau  & \textbf{-0.158}~(\,-0.304,\;0.000\,) & \textbf{-0.121}~(\,\textbf{-0.200},\;\textbf{-0.026}\,) & \textbf{-0.148}~(\,-0.188,\;0.000\,) & \textbf{-0.047}~(\,-0.125,\;0.011\,) \\
HK  & SGD-Step     & \textbf{-0.167}~(\,\textbf{-0.254},\;\textbf{-0.101}\,) & \textbf{-0.125}~(\,-0.214,\;0.000\,) & \textbf{-0.165}~(\,-0.198,\;0.000\,) & \textbf{-0.071}~(\,\textbf{-0.129},\;\textbf{-0.025}\,) \\
HK  & SGD-Constant & 0.000~(\,-0.036,\;0.154\,) & \textbf{-0.014}~(\,-0.113,\;0.011\,) & 0.012~(\,-0.020,\;0.071\,) & 0.015~(\,0.000,\;0.037\,) \\
SWK & DoG          & 0.137~(\,-0.046,\;0.165\,) & \textbf{-0.014}~(\,-0.145,\;0.053\,) & 0.136~(\,0.073,\;0.248\,) & 0.028~(\,-0.078,\;0.122\,) \\
SWK & SGD-Cosine   & \textbf{-0.034}~(\,-0.176,\;0.026\,) & 0.011~(\,-0.041,\;0.072\,) & 0.000~(\,-0.010,\;0.114\,) & \textbf{-0.016}~(\,-0.114,\;0.035\,) \\
SWK & SGD-Exp      & \textbf{-0.152}~(\,\textbf{-0.309},\;\textbf{-0.066}\,) & \textbf{-0.070}~(\,-0.174,\;0.014\,) & \textbf{-0.040}~(\,-0.169,\;0.000\,) & \textbf{-0.023}~(\,-0.286,\;0.000\,) \\
SWK & SGD-Plateau  & \textbf{-0.128}~(\,\textbf{-0.324},\;\textbf{-0.011}\,) & \textbf{-0.049}~(\,-0.151,\;0.012\,) & \textbf{-0.080}~(\,-0.148,\;0.000\,) & \textbf{-0.059}~(\,-0.139,\;0.000\,) \\
SWK & SGD-Step     & \textbf{-0.192}~(\,\textbf{-0.338},\;\textbf{-0.077}\,) & \textbf{-0.049}~(\,-0.126,\;0.012\,) & \textbf{-0.088}~(\,-0.149,\;0.000\,) & \textbf{-0.071}~(\,\textbf{-0.286},\;\textbf{-0.023}\,) \\
SWK & SGD-Constant & 0.011~(\,-0.024,\;0.110\,) & 0.013~(\,-0.013,\;0.027\,) & 0.040~(\,-0.012,\;0.080\,) & 0.000~(\,0.000,\;0.013\,) \\
TOP & DoG          & 0.150~(\,-0.076,\;0.250\,) & \textbf{-0.091}~(\,-0.214,\;0.027\,) & 0.170~(\,0.012,\;0.248\,) & 0.055~(\,-0.013,\;0.132\,) \\
TOP & SGD-Cosine   & \textbf{-0.055}~(\,-0.242,\;0.025\,) & \textbf{-0.014}~(\,-0.076,\;0.030\,) & 0.000~(\,-0.066,\;0.060\,) & 0.000~(\,-0.071,\;0.035\,) \\
TOP & SGD-Exp      & \textbf{-0.152}~(\,\textbf{-0.379},\;\textbf{-0.066}\,) & \textbf{-0.086}~(\,\textbf{-0.266},\;\textbf{-0.027}\,) & \textbf{-0.072}~(\,-0.190,\;0.000\,) & \textbf{-0.026}~(\,-0.154,\;0.000\,) \\
TOP & SGD-Plateau  & \textbf{-0.190}~(\,\textbf{-0.318},\;\textbf{-0.023}\,) & \textbf{-0.121}~(\,-0.200,\;0.000\,) & \textbf{-0.067}~(\,-0.202,\;0.000\,) & \textbf{-0.064}~(\,\textbf{-0.154},\;\textbf{-0.023}\,) \\
TOP & SGD-Step     & \textbf{-0.153}~(\,\textbf{-0.348},\;\textbf{-0.036}\,) & \textbf{-0.061}~(\,\textbf{-0.214},\;\textbf{-0.011}\,) & \textbf{-0.112}~(\,-0.190,\;0.000\,) & \textbf{-0.071}~(\,-0.154,\;0.000\,) \\
TOP & SGD-Constant & 0.000~(\,-0.024,\;0.167\,) & \textbf{-0.025}~(\,-0.109,\;0.025\,) & 0.024~(\,-0.020,\;0.146\,) & 0.000~(\,0.000,\;0.068\,) \\
WD  & DoG          & 0.137~(\,0.074,\;0.209\,) & \textbf{-0.014}~(\,-0.118,\;0.056\,) & 0.136~(\,0.057,\;0.198\,) & 0.060~(\,-0.113,\;0.132\,) \\
WD  & SGD-Cosine   & \textbf{-0.047}~(\,-0.171,\;0.164\,) & 0.014~(\,-0.038,\;0.072\,) & 0.000~(\,-0.010,\;0.114\,) & \textbf{-0.023}~(\,-0.071,\;0.035\,) \\
WD  & SGD-Exp      & \textbf{-0.123}~(\,-0.300,\;0.051\,) & \textbf{-0.070}~(\,-0.118,\;0.014\,) & \textbf{-0.040}~(\,-0.169,\;0.000\,) & \textbf{-0.048}~(\,-0.286,\;0.000\,) \\
WD  & SGD-Plateau  & \textbf{-0.106}~(\,\textbf{-0.324},\;\textbf{-0.023}\,) & \textbf{-0.042}~(\,\textbf{-0.118},\;\textbf{-0.013}\,) & \textbf{-0.080}~(\,-0.148,\;0.000\,) & \textbf{-0.051}~(\,\textbf{-0.286},\;\textbf{-0.016}\,) \\
WD  & SGD-Step     & \textbf{-0.153}~(\,-0.271,\;0.000\,) & \textbf{-0.023}~(\,-0.095,\;0.000\,) & \textbf{-0.088}~(\,-0.149,\;0.000\,) & \textbf{-0.071}~(\,\textbf{-0.274},\;\textbf{-0.023}\,) \\
WD  & SGD-Constant & 0.011~(\,-0.015,\;0.110\,) & 0.013~(\,0.000,\;0.041\,) & 0.040~(\,-0.012,\;0.080\,) & 0.000~(\,-0.097,\;0.013\,) \\
\bottomrule
\end{tabular}
\end{table*}

\subsection{Study 3: Performance Comparisons}
This study evaluates the overall effectiveness of CG-ALR across image and graph classification tasks, benchmarking it against a range of widely-used optimization baselines.
For the image datasets, we train ResNet-18 from scratch; for the graph datasets, we train GCN and GAT from scratch; and then compute connectomes from the neuron co-activations of an FC layer. For the methods described in Section~\ref{sec:study1}, we consider initial learning rates $\eta^\star \in \{0.1,\,0.01,\,0.001\}$, a widely-adopted range for these tasks. For SGD-Constant, we use $\eta \in \{0.1,\,0.01,\,0.001\}$. For DoG, we adopt the default hyperparameters $r_{\varepsilon}=10^{-6}$ and $\varepsilon=10^{-8}$. Further implementation details are provided in the Appendix.

Figure~\ref{fig:study2_val_train_loss} reports training and validation loss for three representative datasets (CIFAR-100, ENZYMES, and PROTEINS), with the remaining datasets shown in the Appendix. CG-ALR consistently accelerates optimization: its training loss nearly always decreases the fastest and reaches the lowest level across schedules. On the two graph datasets, CG-ALR also achieves the lowest validation loss both throughout training and at convergence. Moreover, for CG-ALR the gap between the best and worst initial learning rates (i.e., the strongest and weakest choices within $\{0.1,\,0.01,\,0.001\}$) is consistently smaller than for the baselines, suggesting that CG-ALR provides more robust performance across different initial learning rates.

Since raw performance metrics vary substantially across methods, aggregate comparisons are challenging. To address this, we use the \emph{relation error difference} (RED)~\citep{ivgi2023dog} as a normalized measure of performance differences. Given the test accuracy $acc$ (measured at the epoch with the highest validation accuracy), we define the test error as $err = 1 - acc$. Let $err_{\mathrm{CG-ALR}}$ denote the error when trained with CG-ALR. Then,  
\[
\operatorname{RED}(\mathrm{err}_x,\mathrm{err}_{\mathrm{CG-ALR}}) \;=\; \frac{\mathrm{err}_{\mathrm{CG-ALR}}-\mathrm{err}_x}{\mathrm{err}_{\mathrm{CG-ALR}}},
\]  
where $x$ represents a baseline method.
Negative RED indicates that CG-ALR's generalization performance outperforms the baseline (lower test error, higher test accuracy), while positive RED indicates the opposite.
We also report confidence intervals (CIs) to quantify uncertainty and assess reliability. Specifically, we compute 95\% CIs for the median RED using a seed-level percentile bootstrap with 1,000 resamples (resampling seeds with replacement and recomputing the median). CIs entirely below 0 indicate that CG-ALR is statistically significantly outperforming the baselines, CIs entirely above 0 indicate the baseline is statistically significantly better, and CIs overlapping 0 are inconclusive.

As summarized in Table~\ref{tab:study2_c10_c100} and Table~\ref{tab:study2_enz_pro_gcn_gat_merged}, when CG-ALR outperforms the baselines (negative RED), many CIs lie entirely below 0, indicating statistically significant improvements in generalization. By contrast, in cases where baselines achieve higher average RED values (positive RED), their CIs rarely lie entirely above 0, suggesting that even when baselines appear to generalize better than CG-ALR, the evidence is generally not statistically significant. Overall, CG-ALR achieves superior generalization performance to baselines in most cases, with stronger statistical support for its advantages. Additional experimental results are provided in the Appendix.

\section{Broader Impact}
Our work demonstrates that a network's internal functional connectome can serve as a principled and interpretable signal for guiding optimizer learning-rate decisions, providing an alternative to hand-crafted schedules and loss-only heuristics. By linking topological signals to learning-rate updates, we introduce a mechanism whose behavior can be directly traced to representation reorganization, offering a transparent basis for understanding why the optimizer adjusts step sizes at different stages of training. This interpretability also supports greater trustworthiness, as practitioners can diagnose learning dynamics in terms of measurable structural changes rather than opaque heuristics. Our experiments on image and graph benchmarks validate the effectiveness of this approach, and point toward the broader potential of connectome-guided signals for principled training across diverse domains. Looking forward, this line of research has the potential to make optimization not only more efficient, but also more explainable and reliable, with broad impact on the development of future machine learning systems.

\bibliographystyle{plainnat} 
\bibliography{reference}

\clearpage
\appendix
\thispagestyle{empty}

\onecolumn
\aistatstitle{Supplementary Materials}

\section{RELATED WORK}

Learning rate control in stochastic optimization has long relied on handcrafted schedules such as step decay, exponential decay, cosine annealing, and reduction on plateau; these schedules are primarily time-based or loss-triggered and usually require careful tuning to the model and dataset~\citep{goodfellow2016deep,pytorch2023lr}. Adaptive optimizers such as RMSProp and Adam use gradient statistics to temper this sensitivity; however, they remain hyperparameter dependent and do not directly track how internal representations evolve during training~\citep{tieleman2012rmsprop,kingma2015adam}. Parameter-free methods such as Distance over Gradients (DoG) scale the step size using the ratio of parameter displacement to accumulated gradient norms. This removes explicit schedules while still relying on first-order magnitudes and may miss important phases of internal reorganization~\citep{ivgi2023dog}.

Topological data analysis provides a complementary view. Persistent homology yields stable summaries of structure with Wasserstein and related metrics, and kernelized constructions such as the sliced Wasserstein and heat kernels enable efficient comparisons of evolving topological signals~\citep{cohen2007stability,reininghaus2015stable,carriere2017sliced,skraba2020wasserstein}. Recent work introduces functional connectomes for artificial networks, bringing coactivation graphs and persistent graph homology to the study of representation dynamics~\citep{songdechakraiwut2025functional,songdechakraiwut2023wasserstein}. Our approach builds at this intersection. We couple a Robbins and Monro envelope~\citep{robbins1951sgd} with a topology-driven controller that monitors connectome change through persistent homology distances, for example, Wasserstein comparisons between successive connectomes, so that learning rate adjustments respond to representation reconfiguration rather than to time or loss alone.

\section{TOPOLOGICAL CONNECTOME DISTANCES}
\subsection{Bottleneck Distance (BD)}\label{sec:BD}

The bottleneck distance measures the worst-case discrepancy between two persistence diagrams.  
A persistence diagram $D$ is a multiset of points $(b,d)$ encoding birth and death times of topological features.  
To compare diagrams $D_1$ and $D_2$, consider matchings $\gamma$ that, for every point $p\in D_1\cup D_2$, pair $p$ either with a point in the other diagram or with its projection onto the diagonal $\Delta=\{(x,x)\}$.  
Each point is matched exactly once, and images under $\gamma$ are unique.

The bottleneck distance is
\[
d_B(D_1,D_2)=\inf_{\gamma}\ \sup_{p\in D_1\cup D_2}\ \|p-\gamma(p)\|_{\infty},
\]
where $\|(b_1,d_1)-(b_2,d_2)\|_{\infty}=\max\{|b_1-b_2|,|d_1-d_2|\}$ for point-to-point matches, and for a match to the diagonal one has
\[
\bigl\|(b,d)-\Pi_\Delta(b,d)\bigr\|_{\infty}=\tfrac{1}{2}(d-b),
\]
which is the minimal $\ell_\infty$ distance from $(b,d)$ to $\Delta$.  
Intuitively, $d_B$ reports the largest single adjustment needed to transform one diagram into the other and is a standard stability metric in persistent homology.

\subsection{Heat Kernel (HK)} \label{HK}

The heat kernel defines a positive-definite similarity on persistence diagrams~\citep{reininghaus2015stable}.  
For diagrams $D_1$ and $D_2$, and points $p=(b_p,d_p)\in D_1$, $q=(b_q,d_q)\in D_2$, denote their reflections across the diagonal by $\bar p=(d_p,b_p)$ and $\bar q=(d_q,b_q)$.  
With bandwidth $\sigma>0$, the kernel is
\[
k_\sigma(D_1,D_2)=\frac{1}{8\pi\sigma}\!\!\sum_{\substack{p\in D_1\\ q\in D_2}}\!\!
\bigl[
e^{-\frac{\|p-q\|^2}{8\sigma}}
- e^{-\frac{\|p-\bar q\|^2}{8\sigma}}
- e^{-\frac{\|\bar p-q\|^2}{8\sigma}}
+ e^{-\frac{\|\bar p-\bar q\|^2}{8\sigma}}
\bigr].
\]
The four terms jointly account for interactions with diagonal reflections, which suppress the influence of low-persistence features and yield a symmetric, positive-definite kernel in a reproducing kernel Hilbert space.

\subsection{Sliced Wasserstein Kernel (SWK)} \label{SWK}

The sliced Wasserstein construction combines optimal transport with kernel methods while remaining computationally tractable~\cite{carriere2017sliced}.  
Given diagrams $D_1,D_2$ and order $p\ge 1$, the sliced Wasserstein distance is
\[
\mathrm{SW}_p^p(D_1,D_2)=\frac{1}{\pi}\int_{0}^{\pi} W_p^p\!\bigl(\theta_* D_1,\ \theta_* D_2\bigr)\, d\theta,
\]
where $\theta_*D$ denotes the projection of all points in $D$ onto the line through the origin with angle $\theta$, and $W_p$ is the $p$-Wasserstein distance on the real line.  
A Gaussian kernel built from this distance is
\[
k_{\mathrm{SW}}(D_1,D_2)=\exp\!\left(-\frac{\mathrm{SW}_p^p(D_1,D_2)}{2\tau}\right),
\]
with scale parameter $\tau>0$.

\subsection{Implementation Details for Topological Connectome Distances}
\label{Implementation}

\paragraph{Feature tap and probe set.}
Each epoch uses a fixed probe size $P=\texttt{PROBE\_P}$ for \{\textsc{TOP}, \textsc{WD}, \textsc{SWK}, \textsc{BD}, \textsc{HK}\} (baselines set $P{=}0$).  
Layer/features are obtained via \texttt{activations} (graphs) or \texttt{extract\_features} (images).  
We form a robust absolute correlation $\mathbf{S}\in[0,1]^{P\times P}$ using \texttt{robust\_corr} and zero the diagonal via \texttt{correlation\_graph}.

\paragraph{\textsc{TOP} (graph–filtration vector).}
From $\mathbf{S}$, we compute an MST on $\,\mathbf{1}-\mathbf{S}$, remove its edges, take the strict upper triangle of the remainder, keep positive entries, and sort descending; implemented by \texttt{adj2pers}.

\paragraph{\textsc{WD}/\textsc{SWK}/\textsc{BD}/\textsc{HK} (Vietoris–Rips on $H_1$).}
We set $\mathbf{D}=\mathbf{1}-|\mathrm{corr}|$ (zero diagonal) and run Vietoris–Rips with \texttt{compute\_persistence\_diagram\_from\_correlation} (\(H_1\) only).  
Epoch distances use \texttt{persim\_wasserstein} (\textsc{WD}, $p=\texttt{pd\_w\_p}$), \texttt{sliced\_wasserstein} (\textsc{SWK}, $M=\texttt{swk\_K}$, $p=\texttt{swk\_p}$), \texttt{bottleneck} (\textsc{BD}), and \texttt{heat} (\textsc{HK}); we fix \(H_1\) with unit weight.

\paragraph{Topological time series.}
Let $\mathcal{S}_e$ be the epoch-$e$ signature (GF vector or \(H_1\) PD).  
We set $W_0{=}0$ and for \textsc{TOP} compute $W_e = W_1(\text{sorted GF}_e,\text{sorted GF}_{e-1})$ via \texttt{wasserstein\_distance};  
for PDs, $W_e$ is the chosen \texttt{persim} distance above.  
The series $\{W_e\}$ is consumed within \texttt{train\_model}.

\section{THEOREM 1 PROOFS}

\paragraph{Setup.}
We minimize the population objective
\[
f(\theta)\coloneqq \mathbb{E}_{\xi}[\ell(\theta;\xi)]
\]
over $\theta\in\mathbb{R}^d$. At (global) SGD step $t\in\mathbb{N}_0$, the iterate is updated by
\begin{equation}\label{eq:update-app}
\theta_{t+1}=\theta_t-\eta_t\,g_t,\qquad 
g_t=\nabla \ell(\theta_t;\xi_t),
\end{equation}
with a stepsize of the form
\begin{equation}\label{eq:stepsize-app}
\eta_t=\psi_{e(t)-1}^{-1}\,\tilde\eta_t,
\qquad 
\tilde\eta_t=\frac{\eta_0}{(t+t_0)^\alpha},
\end{equation}
where $\eta_0>0$, $t_0\ge 1$, and $\alpha\in(1/2,1]$. Training is organized in epochs indexed by $e\in\mathbb{N}_0$, with integer boundaries
$0=T_0<T_1<T_2<\cdots$ and epoch membership map $e(t)$ defined by $e(t)=e$ iff $T_e\le t<T_{e+1}$. 
At all $t$ in epoch $e\ge 1$, the (clipped) factor $\psi_{e-1}$ computed from epoch $e-1$ is used; for $e=0$ we set a deterministic $\psi_{-1}\in[\psi_{\min},\psi_{\max}]$.
Let $(\mathcal{F}_t)_{t\ge 0}$ be the canonical filtration
\[
\mathcal{F}_t \coloneqq \sigma\!\big(\theta_0,\ \xi_0,\ldots,\xi_{t-1}\big),
\]
so that $g_t$ depends on $\xi_t$ and is \emph{not} $\mathcal{F}_t$-measurable, while any quantity decided before sampling $\xi_t$ must be $\mathcal{F}_t$-measurable (``predictable'').

\paragraph{Assumptions.}
\begin{enumerate}
\item[(A1)] (Lower boundedness) $f_\star\coloneqq \inf_{\theta}f(\theta)>-\infty$.
\item[(A2)] (Smoothness) $f$ is $L$-smooth: for all $x,y$, 
$\|\nabla f(x)-\nabla f(y)\|\le L\|x-y\|$.
\item[(A3)] (Unbiased stochastic gradient with bounded second moment) For some constants $A,B\ge 0$,
\[
\mathbb{E}[g_t\mid \mathcal{F}_t]=\nabla f(\theta_t),\qquad
\mathbb{E}\!\big[\|g_t\|^2\mid \mathcal{F}_t\big]\le A+B\,\|\nabla f(\theta_t)\|^2.
\]
\item[(A4)] (Bounded iterates) $(\theta_t)_{t\ge 0}$ is almost surely bounded.
\item[(A5)] (Predictable, bounded controller) For every epoch $e\ge 0$,
$\psi_e$ is $\mathcal{F}_{T_{e+1}}$-measurable and satisfies 
$0<\psi_{\min}\le \psi_e\le \psi_{\max}<\infty$ almost surely.
Equivalently, for any $t$ with $e(t)=e+1$, $\psi_e$ is $\mathcal{F}_t$-measurable.
\item[(A6)] (Robbins--Monro envelope) $\alpha\in(1/2,1]$ in \eqref{eq:stepsize-app}.
\end{enumerate}

\begin{remark}
The analysis uses only (A5), i.e., that the controller produces an \emph{a priori} bounded, predictable multiplicative factor per epoch. The specific statistic used to form $\psi_e$ (e.g., a 1D Wasserstein distance between empirical similarity distributions on $[0,1]$) is immaterial to the proofs.
\end{remark}

\paragraph{Technical lemmas.}
\begin{lemma}[Robbins--Monro conditions preserved]\label{lem:RM-app}
Under \textnormal{(A5)--(A6)}, the stepsizes are predictable and satisfy
\[
\sum_{t=0}^\infty \eta_t=\infty,\qquad 
\sum_{t=0}^\infty \eta_t^2<\infty,
\qquad \eta_t\ \text{is } \mathcal{F}_t\text{-measurable for all }t,\qquad \eta_t\to 0.
\]
\end{lemma}
\begin{proof}
By (A5), $\psi_{e(t)-1}\in[\psi_{\min},\psi_{\max}]$ and is decided before sampling $\xi_t$, hence $\mathcal{F}_t$-measurable. Thus $\eta_t$ is predictable as the product of an $\mathcal{F}_t$-measurable random variable and the deterministic $\tilde\eta_t$. Comparison gives
\[
\frac{1}{\psi_{\max}}\sum_t \tilde\eta_t \ \le\ \sum_t \eta_t \ \le\ \frac{1}{\psi_{\min}}\sum_t \tilde\eta_t,
\qquad
\sum_t \eta_t^2 \ \le\ \frac{1}{\psi_{\min}^2}\sum_t \tilde\eta_t^2.
\]
Since $\alpha\in(1/2,1]$, $\sum_t \tilde\eta_t=\infty$ and $\sum_t \tilde\eta_t^2<\infty$; the same then holds for $(\eta_t)$, and moreover $\eta_t\to 0$.
\end{proof}

\begin{lemma}[One-step expected descent]\label{lem:descent-app}
Under \textnormal{(A2)--(A3)}, the update \eqref{eq:update-app} satisfies
\[
\mathbb{E}\!\left[f(\theta_{t+1})\mid \mathcal{F}_t\right]
\le f(\theta_t) 
- \left(\eta_t-\tfrac{LB}{2}\eta_t^2\right)\!\|\nabla f(\theta_t)\|^2
+ \tfrac{LA}{2}\eta_t^2.
\]
\end{lemma}
\begin{proof}
$L$-smoothness gives 
$f(\theta_{t+1})\le f(\theta_t)-\eta_t\langle \nabla f(\theta_t),g_t\rangle
+\tfrac{L}{2}\eta_t^2\|g_t\|^2$.
Condition on $\mathcal{F}_t$ and use (A3).
\end{proof}

\begin{lemma}[Robbins--Siegmund, almost supermartingale]\label{lem:RS-app}
Let $(X_t)$ be nonnegative and adapted. Suppose there exist nonnegative adapted $(Y_t)$ and $(Z_t)$ such that
\[
\mathbb{E}[X_{t+1}\mid\mathcal{F}_t]\le X_t - Y_t + Z_t,
\qquad \sum_{t=0}^{\infty}\mathbb{E}[Z_t]<\infty.
\]
Then $(X_t)$ converges almost surely to a finite random variable and $\sum_{t=0}^{\infty}Y_t<\infty$ almost surely.
\end{lemma}
\vspace{-0.5\baselineskip}

\paragraph{Main results.}
\begin{theorem}[Convergence to stationary points]\label{thm:stationary-app}
Under \textnormal{(A1)--(A6)}, the following hold:
\begin{enumerate}
\item $f(\theta_t)$ converges almost surely to a finite random variable;
\item $\displaystyle \sum_{t=0}^\infty \eta_t\|\nabla f(\theta_t)\|^2<\infty$ almost surely, hence $\liminf_{t\to\infty}\|\nabla f(\theta_t)\|=0$ almost surely;
\item Every almost-sure limit point of $(\theta_t)$ is a stationary point of $f$.
\end{enumerate}
\end{theorem}

\begin{proof}
Define
\[
X_t\coloneqq f(\theta_t)-f_\star\ge 0,\quad
Y_t\coloneqq\Big(\eta_t-\tfrac{LB}{2}\eta_t^2\Big)\,\|\nabla f(\theta_t)\|^2,\quad
Z_t\coloneqq \tfrac{LA}{2}\eta_t^2.
\]
By Lemmas~\ref{lem:descent-app} and~\ref{lem:RM-app}, 
$\mathbb{E}[X_{t+1}\mid\mathcal{F}_t]\le X_t-Y_t+Z_t$ and $\sum_t \mathbb{E}[Z_t]\le\frac{LA}{2}\sum_t \eta_t^2<\infty$. Since $\eta_t\to 0$, there exists an almost surely finite time $T$ such that for all $t\ge T$, $\eta_t-\tfrac{LB}{2}\eta_t^2\ge \tfrac{1}{2}\eta_t$. Applying Lemma~\ref{lem:RS-app} from $T$ onward (the finite prefix being harmless) yields the convergence of $(f(\theta_t))_t$ and $\sum_t Y_t<\infty$ almost surely. In particular, $\sum_t \eta_t\|\nabla f(\theta_t)\|^2<\infty$ almost surely.

By Lemma~\ref{lem:RM-app}, $\sum_t \eta_t=\infty$; thus the only way for $\sum_t \eta_t\|\nabla f(\theta_t)\|^2$ to be finite is that $\liminf_t \|\nabla f(\theta_t)\|=0$ almost surely. Boundedness of $(\theta_t)$ (A4) implies the existence of almost-sure accumulation points; by continuity of $\nabla f$ and the previous display, every limit point $\bar\theta$ satisfies $\nabla f(\bar\theta)=0$.
\end{proof}

\begin{theorem}[Convergence under PL]\label{thm:PL-app}
Suppose, in addition, that $f$ satisfies the Polyak--\L{}ojasiewicz (PL) inequality
\[
\frac{1}{2}\|\nabla f(\theta)\|^2\ \ge\ \mu\big(f(\theta)-f_\star\big)\qquad\text{for some }\mu>0\text{ and all }\theta.
\]
Then $f(\theta_t)\to f_\star$ almost surely and 
$\mathrm{dist}(\theta_t,\mathcal{X}_\star)\to 0$ almost surely,
where $\mathcal{X}_\star\coloneqq \arg\min f$. Consequently, every limit point of $(\theta_t)$ is a (global) minimizer. If $\mathcal{X}_\star$ is a singleton, then $\theta_t\to \theta_\star$ almost surely.
\end{theorem}

\begin{proof}
Under PL,
$\|\nabla f(\theta_t)\|^2\ge 2\mu\,(f(\theta_t)-f_\star)$, hence
\[
\sum_t \eta_t\,(f(\theta_t)-f_\star)
\ \le\ \frac{1}{2\mu}\sum_t \eta_t\|\nabla f(\theta_t)\|^2
\ <\ \infty\quad\text{a.s.}
\]
By Theorem~\ref{thm:stationary-app}, $f(\theta_t)$ converges almost surely and $\sum_t \eta_t=\infty$ (Lemma~\ref{lem:RM-app}); therefore $f(\theta_t)\to f_\star$ almost surely (otherwise the above sum would diverge). PL then implies $\|\nabla f(\theta_t)\|\to 0$ almost surely and, by standard error-bound arguments under PL (e.g., quadratic growth), $\mathrm{dist}(\theta_t,\mathcal{X}_\star)\to 0$ almost surely. If $\mathcal{X}_\star$ is a singleton, the whole sequence converges.
\end{proof}

\paragraph{Why the controller cannot break convergence.}
The controller forms $\psi_e$ from epoch-$e$ data and \emph{clips} it into $[\psi_{\min},\psi_{\max}]$. Thus $\psi_e$ is $\mathcal{F}_{T_{e+1}}$-measurable and uniformly bounded (A5). Any additional heuristics (e.g., cooldowns, multiplicative up/down rules, late-phase boosts) that transform $\psi_e$ by $\mathcal{F}_{T_{e+1}}$-measurable operations and re-clip to $[\psi_{\min},\psi_{\max}]$ preserve (A5). Consequently, Lemma~\ref{lem:RM-app} holds, which is the only property of the controller used by the proofs of Theorems~\ref{thm:stationary-app} and~\ref{thm:PL-app}.

\paragraph{Notes on predictability.}
At each step $t$, the random variables $\psi_{e(t)-1}$ and $\eta_t$ are $\mathcal{F}_t$-measurable (decided before sampling $\xi_t$), while $g_t$ is only $\mathcal{F}_{t+1}$-measurable. This ensures the conditional expectations in Lemma~\ref{lem:descent-app} are well-defined and that the Robbins--Siegmund lemma applies with $Z_t=\tfrac{LA}{2}\eta_t^2$ and $\sum_t \mathbb{E}[Z_t]<\infty$ by Lemma~\ref{lem:RM-app}.

\section{EXPERIMENT DETAILS}
\paragraph{Additional Dataset Details.}
MUTAG consists of graphs representing nitroaromatic compounds, labeled according to whether they exhibit mutagenic effect on the \textit{Salmonella typhimurium} bacterium. PROTEINS contains protein graphs labeled as enzymes or non-enzymes. ENZYMES includes protein graphs categorized into 6 predefined classes according to enzyme commission numbers \citep{Morris2020TUDataset}.

\begin{table*}[t]
\centering
\small
\caption{Training configurations for CIFAR-10, CIFAR-100, and Mini-ImageNet experiments.}
\vspace{0.6\baselineskip}
\label{tab:image_params}
\begin{tabular}{lllll}
\toprule
\textbf{Parameter} & \textbf{Symbol / Key} & \textbf{CIFAR-10} & \textbf{CIFAR-100} & \textbf{Mini-ImageNet} \\
\midrule
Epochs & \texttt{epochs} & 50 & 90 & 90 \\
Optimizer & \texttt{optimizer} & SGD & SGD & SGD \\
Exponential decay $\gamma$ & \texttt{exp\_gamma} & 0.97 & 0.97 & 0.97 \\
Weight decay & \texttt{weight\_decay} & $5\times 10^{-4}$ & $7\times 10^{-4}$ & $5\times 10^{-4}$ \\
Momentum & \texttt{momentum} & 0.9 & 0.9 & 0.9 \\
\midrule
RM offset & $t_0$ (\texttt{t0}) & 1600 & 2000 & 1600 \\
RM exponent & $\alpha$ (\texttt{alpha}) & 0.52 & 0.50 & 0.60 \\
Upscale multiplier & $\gamma_{\uparrow}$ (\texttt{gamma\_up}) & 1.20 & 1.22 & 1.08 \\
Downscale multiplier & $\gamma_{\downarrow}$ (\texttt{gamma\_down}) & 0.85 & 0.88 & 0.82 \\
Late-phase multiplier & $\gamma_{\mathrm{late}}$ (\texttt{gamma\_late}) & 0.985 & 0.98 & 0.95 \\
Cooldown length & \texttt{cooldown} & 4 & 5 & 3 \\
Trigger count & \texttt{n\_trigger} & 6 & 7 & 3 \\
Late split ratio & $N_{\mathrm{ratio}}$ (\texttt{N\_ratio}) & 0.88 & 0.90 & 0.76 \\
Probe set size & $P_{\text{probe}}$ (\texttt{probe\_P}) & 1024 & 1024 & 10000 \\
$\psi_{\min}$ & \texttt{psi\_min} & 0.65 & 0.70 & 0.60 \\
$\psi_{\max}$ & \texttt{psi\_max} & 6.0 & 10.0 & 1.70 \\
Smoothing factor & $\beta$ (\texttt{beta}) & 0.96 & 0.97 & 0.94 \\
Threshold step size & $\tau$ (\texttt{tau}) & 0.002 & 0.003 & 0.002 \\
Robust window & $w$ (\texttt{robust\_w}) & 13 & 15 & 17 \\
MAD multiplier & $k$ (\texttt{mad\_k}) & 3.6 & 3.8 & 3.2 \\
Warm-up epochs & $K_{\mathrm{warm}}$ (\texttt{K\_warm}) & 4 & 6 & 12 \\

\bottomrule
\end{tabular}
\end{table*}

\begin{table*}[t]
\centering
\small
\caption{Training configurations for MUTAG, PROTEINS, and ENZYMES experiments.}
\vspace{0.6\baselineskip}
\label{tab:graph_params}
\begin{tabular}{lllll}
\toprule
\textbf{Parameter} & \textbf{Symbol / Key} & \textbf{MUTAG} & \textbf{PROTEINS} & \textbf{ENZYMES} \\
\midrule
Epochs & \texttt{epochs} & 300 & 500 & 500 \\
Optimizer & \texttt{optimizer} & SGD & SGD & SGD \\
Exponential decay $\gamma$ & \texttt{exp\_gamma} & 0.97 & 0.97 & 0.97 \\
Weight decay & \texttt{weight\_decay} & $5\times 10^{-4}$ & $5\times 10^{-4}$ & $5\times 10^{-4}$ \\
Momentum & \texttt{momentum} & 0.9 & 0.9 & 0.9 \\
\midrule
RM offset & $t_0$ (\texttt{t0}) & 800 & 1300 & 1200 \\
RM exponent & $\alpha$ (\texttt{alpha}) & 0.56 & 0.56 & 0.57 \\
Upscale multiplier & $\gamma_{\uparrow}$ (\texttt{gamma\_up}) & 1.10 & 1.08 & 1.08 \\
Downscale multiplier & $\gamma_{\downarrow}$ (\texttt{gamma\_down}) & 0.80 & 0.84 & 0.84 \\
Late-phase multiplier & $\gamma_{\mathrm{late}}$ (\texttt{gamma\_late}) & 0.95 & 0.96 & 0.96 \\
Cooldown length & \texttt{cooldown} & 3 & 4 & 4 \\
Trigger count & \texttt{n\_trigger} & 3 & 4 & 4 \\
Late split ratio & $N_{\mathrm{ratio}}$ (\texttt{N\_ratio}) & 0.70 & 0.80 & 0.78 \\
$\psi_{\min}$ & \texttt{psi\_min} & 0.62 & 0.60 & 0.60 \\
$\psi_{\max}$ & \texttt{psi\_max} & 1.8 & 1.8 & 2.0 \\
Smoothing factor & $\beta$ (\texttt{beta}) & 0.94 & 0.95 & 0.95 \\
Threshold step size & $\tau$ (\texttt{tau}) & 0.001 & 0.002 & 0.002 \\
Robust window & $w$ (\texttt{robust\_w}) & 12 & 16 & 15 \\
MAD multiplier & $k$ (\texttt{mad\_k}) & 3.3 & 3.2 & 3.2 \\
Warm-up epochs & $K_{\mathrm{warm}}$ (\texttt{K\_warm}) & 12 & 16 & 16 \\
\bottomrule
\end{tabular}
\end{table*}

\begin{figure*}[t]
  \centering
  \includegraphics[width=\linewidth]{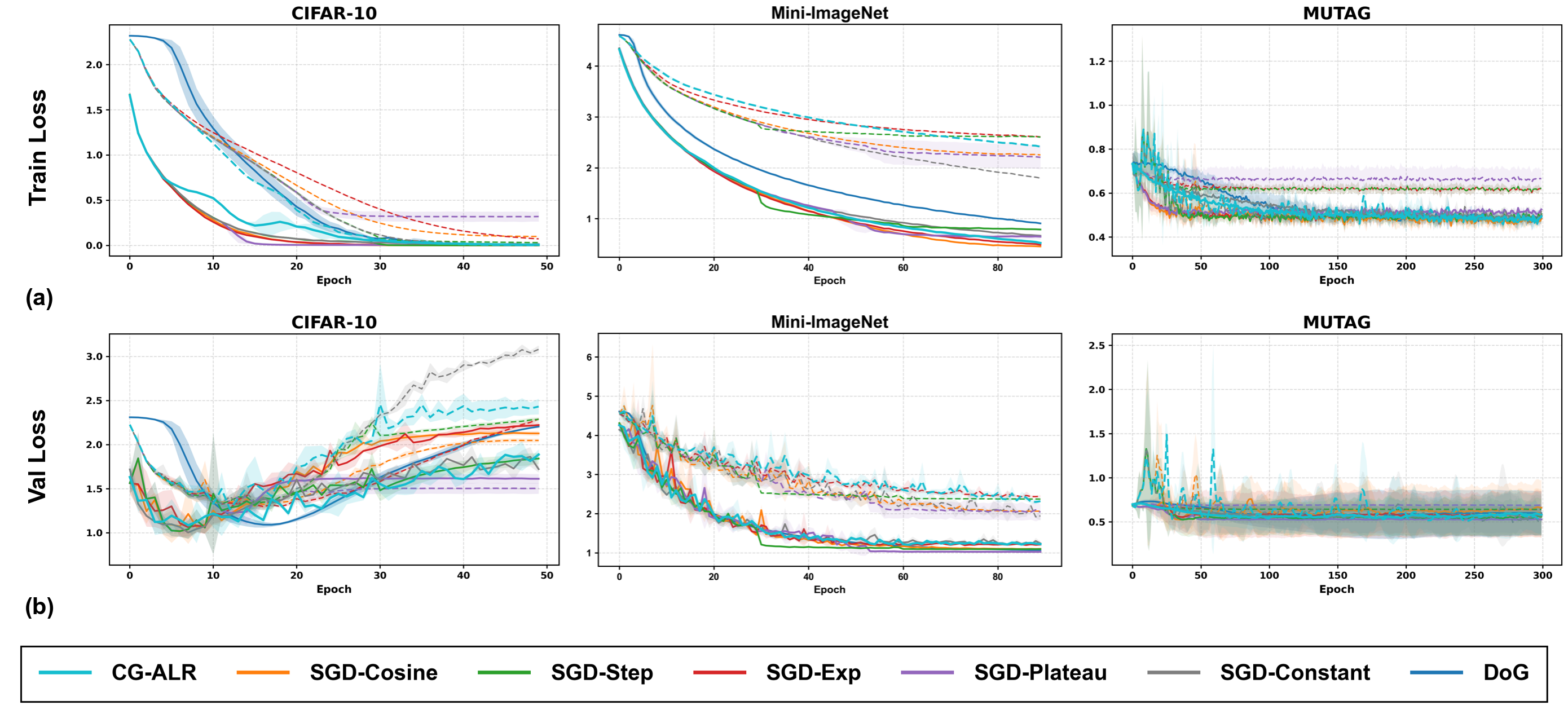}
  \vspace{0.6ex}
  \caption{
  (a) \textbf{Training Loss} on three datasets (\textbf{CIFAR-10}, \textbf{Mini-ImageNet}, and \textbf{MUTAG}). 
  (b) \textbf{Validation Loss} on the same datasets. 
  Solid lines indicate the averages across three seeds using the \emph{best} learning-rate configuration from the predefined grid 
  ($\eta^\star\!\in\!\{0.1,0.01,0.001\}$ for CG-ALR and SGD-based schedules; 
  $\eta\!\in\!\{0.1,0.01,0.001\}$ for SGD-Constant; 
  and default hyperparameters for DoG). 
  Dashed lines denote the corresponding \emph{worst} configurations.
  Shaded regions represent the min–max range across seeds.
  All methods share the same model backbone and training protocol.
  }
  \label{fig:appendix_train_val_loss}
\end{figure*}

\begin{table*}[t]
\centering
\caption{Median RED with 95\% Percentile-Bootstrap CIs Across Seeds for Mini-ImageNet (ResNet-18). 
\textbf{Bold} entries indicate cases where CG-ALR achieves better generalization performance ($\mathrm{RED}<0$), 
or statistically significant improvements when the 95\% CI lies entirely below 0.}
\label{tab:study2_miniimagenet}
\vspace{1.5em}
\normalsize
\setlength{\tabcolsep}{4pt}
\begin{tabular}{@{}llr@{}}
\toprule
\multicolumn{1}{c}{\textbf{CG-ALR}} &
\multicolumn{1}{c}{\textbf{Baseline}} &
\multicolumn{1}{c}{\textbf{RED (95\% CI)}} \\
\midrule
SWK & DoG          & 0.074~(\,-0.129,\;0.487\,) \\
SWK & SGD-Cosine   & 0.085~(\,0.026,\;0.107\,) \\
SWK & SGD-Exp      & \textbf{-0.018}~(\,\textbf{-0.041},\;\textbf{-0.008}\,) \\
SWK & SGD-Plateau  & 0.052~(\,-0.004,\;0.157\,) \\
SWK & SGD-Step     & \textbf{-0.027}~(\,\textbf{-0.048},\;\textbf{-0.001}\,) \\
SWK & SGD-Constant & 0.033~(\,-0.046,\;0.184\,) \\
\midrule
TOP & DoG          & 0.074~(\,-0.129,\;0.487\,) \\
TOP & SGD-Cosine   & 0.084~(\,0.026,\;0.107\,) \\
TOP & SGD-Exp      & \textbf{-0.021}~(\,\textbf{-0.041},\;\textbf{-0.018}\,) \\
TOP & SGD-Plateau  & 0.052~(\,-0.004,\;0.156\,) \\
TOP & SGD-Step     & \textbf{-0.028}~(\,\textbf{-0.048},\;\textbf{-0.008}\,) \\
TOP & SGD-Constant & 0.033~(\,-0.043,\;0.176\,) \\
\midrule
WD  & DoG          & 0.074~(\,-0.129,\;0.487\,) \\
WD  & SGD-Cosine   & 0.076~(\,0.026,\;0.107\,) \\
WD  & SGD-Exp      & \textbf{-0.021}~(\,\textbf{-0.046},\;\textbf{-0.008}\,) \\
WD  & SGD-Plateau  & 0.052~(\,-0.004,\;0.144\,) \\
WD  & SGD-Step     & \textbf{-0.028}~(\,\textbf{-0.048},\;\textbf{-0.022}\,) \\
WD  & SGD-Constant & 0.033~(\,-0.046,\;0.183\,) \\
\bottomrule
\end{tabular}
\end{table*}

\paragraph{System Configuration.}
All experiments were performed on the system with the following specifications:
\begin{itemize}
  \item \textbf{CPU:} Intel Xeon Gold 6226, 2.70\,GHz
  \item \textbf{GPU:} NVIDIA RTX 5000 Ada Generation, 32\,GB VRAM
  \item \textbf{DRAM:} 755\,GB
  \item \textbf{Architecture:} x86\_64
  \item \textbf{OS:} Ubuntu 22.04.5 LTS
 \item \textbf{Kernel:} Linux 5.15
  \item \textbf{CUDA Toolkit:} 12.8
  \item \textbf{NVIDIA Driver:} 570.148.08
\end{itemize}

\paragraph{Hyperparameters in Algorithm 1.}
The hyperparameters used in all experiments are listed in Table~\ref{tab:image_params}.

\section{ADDITIONAL EXPERIMENT RESULTS}

\begin{table*}[t]
\centering
\caption{Median RED with 95\% Percentile-Bootstrap CIs Across Seeds. 
\textbf{Bold} entries indicate cases where CG-ALR achieves better generalization performance ($\mathrm{RED}<0$), 
or statistically significant improvements when the 95\% CI lies entirely below 0. 
Results are shown for \textbf{MUTAG} with GCN and GAT.}
\label{tab:study2_mutag_gcn_gat_merged}
\vspace{1.2em}
\small
\setlength{\tabcolsep}{3pt}
\begin{tabular}{@{}llrr@{}}
\toprule
\multicolumn{2}{c}{} & \multicolumn{1}{c}{\textbf{GCN}} & \multicolumn{1}{c}{\textbf{GAT}} \\
\cmidrule(lr){3-4}
\multicolumn{1}{c}{\textbf{CG-ALR}} &
\multicolumn{1}{c}{\textbf{Baseline}} &
\multicolumn{1}{c}{RED (95\% CI)} &
\multicolumn{1}{c}{RED (95\% CI)} \\
\midrule
BD  & DoG          & \textbf{-0.111}~(\,\textbf{-0.500},\;\textbf{-0.091}\,) &  0.000~(\,-0.625,\; 0.364\,) \\
BD  & SGD-Cosine   &  0.111~(\,-0.500,\; 0.182\,) & \textbf{-0.182}~(\,-0.429,\; 0.000\,) \\
BD  & SGD-Exp      &  0.111~(\,-0.500,\; 0.182\,) & \textbf{-0.182}~(\,-0.300,\; 0.077\,) \\
BD  & SGD-Plateau  & \textbf{-0.333}~(\,-0.500,\; 0.182\,) & \textbf{-0.182}~(\,-0.857,\; 0.000\,) \\
BD  & SGD-Step     &  0.111~(\,-0.500,\; 0.182\,) & \textbf{-0.182}~(\,-0.857,\; 0.000\,) \\
BD  & SGD-Constant & \textbf{-0.333}~(\,-0.500,\; 0.182\,) &  0.000~(\,-0.004,\; 0.091\,) \\
\midrule
HK  & DoG          &  0.000~(\,-0.200,\; 0.077\,) &  0.000~(\,-0.625,\; 0.462\,) \\
HK  & SGD-Cosine   &  0.200~(\,-0.200,\; 0.308\,) & \textbf{-0.100}~(\,-0.250,\; 0.000\,) \\
HK  & SGD-Exp      &  0.200~(\,-0.200,\; 0.308\,) & \textbf{-0.182}~(\,-0.300,\; 0.077\,) \\
HK  & SGD-Plateau  & \textbf{-0.200}~(\,-0.200,\; 0.308\,) & \textbf{-0.182}~(\,-0.308,\; 0.000\,) \\
HK  & SGD-Step     &  0.200~(\,-0.200,\; 0.308\,) & \textbf{-0.182}~(\,-0.625,\; 0.000\,) \\
HK  & SGD-Constant & \textbf{-0.200}~(\,-0.200,\; 0.308\,) &  0.000~(\, 0.000,\; 0.182\,) \\
\midrule
SWK & DoG          &  0.000~(\,-0.333,\; 0.231\,) &  0.000~(\,-0.625,\; 0.000\,) \\
SWK & SGD-Cosine   &  0.000~(\, 0.000,\; 0.385\,) & \textbf{-0.182}~(\,-0.429,\; 0.000\,) \\
SWK & SGD-Exp      &  0.000~(\, 0.000,\; 0.385\,) & \textbf{-0.182}~(\,-0.300,\; 0.077\,) \\
SWK & SGD-Plateau  &  0.000~(\, 0.000,\; 0.077\,) & \textbf{-0.182}~(\,-0.857,\; 0.000\,) \\
SWK & SGD-Step     &  0.000~(\, 0.000,\; 0.385\,) & \textbf{-0.182}~(\,-0.857,\; 0.000\,) \\
SWK & SGD-Constant &  0.000~(\, 0.000,\; 0.077\,) &  0.000~(\, 0.000,\; 0.000\,) \\
\midrule
TOP & DoG          & \textbf{-0.250}~(\,-0.500,\; 0.000\,) &  0.000~(\,-0.625,\; 0.304\,) \\
TOP & SGD-Cosine   &  0.000~(\,-0.125,\; 0.000\,) & \textbf{-0.100}~(\,-0.375,\; 0.000\,) \\
TOP & SGD-Exp      &  0.000~(\,-0.125,\; 0.000\,) & \textbf{-0.300}~(\,-0.429,\; 0.077\,) \\
TOP & SGD-Plateau  & \textbf{-0.125}~(\,-0.500,\; 0.000\,) & \textbf{-0.300}~(\,-0.625,\; 0.000\,) \\
TOP & SGD-Step     &  0.000~(\,-0.125,\; 0.000\,) & \textbf{-0.300}~(\,-0.625,\; 0.000\,) \\
TOP & SGD-Constant & \textbf{-0.125}~(\,-0.500,\; 0.000\,) &  0.000~(\, 0.000,\; 0.100\,) \\
\midrule
WD  & DoG          & \textbf{-0.200}~(\,-0.200,\; 0.000\,) &  0.000~(\,-0.625,\; 0.300\,) \\
WD  & SGD-Cosine   &  0.100~(\,-0.200,\; 0.200\,) & \textbf{-0.250}~(\,-0.429,\; 0.000\,) \\
WD  & SGD-Exp      &  0.100~(\,-0.200,\; 0.200\,) & \textbf{-0.250}~(\,-0.300,\; 0.077\,) \\
WD  & SGD-Plateau  & \textbf{-0.200}~(\,-0.200,\; 0.100\,) & \textbf{-0.300}~(\,-0.857,\; 0.000\,) \\
WD  & SGD-Step     &  0.100~(\,-0.200,\; 0.200\,) & \textbf{-0.300}~(\,-0.857,\; 0.000\,) \\
WD  & SGD-Constant & \textbf{-0.200}~(\,-0.200,\; 0.100\,) &  0.000~(\,-0.004,\; 0.000\,) \\
\bottomrule
\end{tabular}
\end{table*}

Figure~\ref{fig:appendix_train_val_loss} reports training and validation loss for additional datasets (CIFAR-10, Mini-ImageNet, MUTAG).  For MUTAG, we employ a single fully connected (FC) layer and compute the connectomes from this first FC layer. For CIFAR-10 and Mini-ImageNet, the models include three FC layers, and the connectomes are computed from the second FC layer.

For Mini-ImageNet, as shown in Table~\ref{tab:study2_miniimagenet}, CG-ALR delivers statistically significant generalization gains in several comparisons. For MUTAG with GCN and GAT backbones, we employ a single FC layer and compute the connectomes from this first FC layer. This design choice follows from the small graph size and limited sample count of MUTAG (\(n = 188\)), where deeper architectures tend to overfit rapidly and provide limited representational diversity. As shown in Table~\ref{tab:study2_mutag_gcn_gat_merged}, our CG-ALR outperforms the baselines, though not always with statistical significance. This is largely due to the simplicity and small scale of MUTAG, where performance differences are easily saturated and variance across seeds can obscure consistent trends.

\end{document}